\pdfoutput=1

\documentclass[letterpaper, 10 pt, conference]{ieeeconf}  %

\IEEEoverridecommandlockouts                              %

\usepackage{times}
\usepackage{cite}

\usepackage{adjustbox}
\usepackage{placeins}

\usepackage{amsmath}
\usepackage{amssymb}
\usepackage{mathtools}
\usepackage{siunitx}
\DeclareSIUnit{\year}{yr}
\usepackage{stackengine}
\usepackage{multicol}
\usepackage[bookmarks=true]{hyperref}

\usepackage{graphicx}

\usepackage{aliascnt}
\usepackage[capitalise]{cleveref}
\usepackage{url}

\usepackage{algorithm}
\usepackage{algpseudocode}
\usepackage{booktabs}
\usepackage{array}
\usepackage{aligned-overset}
\usepackage{balance} %

\usepackage{rotating}
\usepackage{multirow}
\usepackage{diagbox}

\usepackage[skins]{tcolorbox}
\newtcolorbox[auto counter, crefname={Problem}{Problems}]{problembox}[1][]{%
	colback=gray!8,
	colframe=black!70,
	boxrule=0.4pt,
	arc=1pt,
	left=4pt, right=4pt, top=2pt, bottom=2pt,
	fonttitle=\bfseries,
	title={Problem~\thetcbcounter},
	#1
}

\DeclareMathOperator{\occ}{occ}

\newcommand{\Rbb}{\ensuremath{\mathbb{R}}}

\newcommand{\tb}{\ensuremath{t_{M}}}

\newcommand{\db}{\ensuremath{\boldsymbol{d}}}

\newcommand{\xb}{\ensuremath{\boldsymbol{x}}}
\newcommand{\yb}{\ensuremath{\boldsymbol{y}}}
\newcommand{\zb}{\ensuremath{\boldsymbol{z}}}

\newcommand{\Cb}{\ensuremath{\boldsymbol{C}}}

\newcommand{\Ib}{\ensuremath{\boldsymbol{I}}}

\newcommand{\Xs}{\ensuremath{\mathcal{X}}}

\newcommand{\chix}[1]{\ensuremath{\chi_{\text{#1}}}}

\newcommand{\pb}{\ensuremath{\boldsymbol{p}}}
\newcommand{\vb}{\ensuremath{\boldsymbol{v}}}
\newcommand{\rob}[1]{\ensuremath{{#1}^{r}}}
\newcommand{\robi}[1]{\ensuremath{\rob{{#1}}_{i}}}

\newcommand{\Orobia}{\ensuremath{\robi{\mathcal{O}}}}

\newcommand{\hum}[1]{\ensuremath{{#1}^{h}}}

\newcommand{\Bs}{\ensuremath{\mathcal{B}}}

\newaliascnt{proposition}{theorem}
\newtheorem{proposition}[proposition]{Proposition}

\aliascntresetthe{proposition}

\crefname{section}{Sec.}{Sections}
\Crefname{section}{Sec.}{Sections}
\crefname{ALG@line}{line}{lines}
\Crefname{ALG@line}{Line}{Lines}

\newcommand{\shield}{SARA shield}

\makeatletter
\newcommand\footnoteref[1]{\protected@xdef\@thefnmark{\ref{#1}}\@footnotemark}
\newcommand*{\rom}[1]{\text{\expandafter\@slowromancap\romannumeral #1@}}
\makeatother

\usepackage{ifthen}
\newboolean{anon}
\newboolean{arxiv}
\setboolean{arxiv}{true}
\setboolean{anon}{false}

\ifthenelse{\boolean{arxiv}}{
}{
}

\ifthenelse{\boolean{anon}}{
}{
}

\ifthenelse{\boolean{anon}}{
	\hypersetup{
		pdftitle={Vision-Based Safe Human-Robot Collaboration with Uncertainty Guarantees},
		pdfauthor={},
		pdfsubject={Human-Robot Collaboration},
		pdfkeywords={Human-Robot Collaboration; Safety; Formal Methods; Human Pose Estimation; Human Motion Prediction}
	}
}{
	\hypersetup{
		pdftitle={Vision-Based Safe Human-Robot Collaboration with Uncertainty Guarantees},
		pdfauthor={Jakob Thumm},
		pdfsubject={Human-Robot Collaboration},
		pdfkeywords={Human-Robot Collaboration; Safety; Formal Methods; Human Pose Estimation; Human Motion Prediction}
	}
}

\title{\LARGE \bf
Vision-Based Safe Human-Robot Collaboration with Uncertainty Guarantees
}

\ifthenelse{\boolean{anon}}{
\author{Author Names Omitted for Anonymous Review. Paper-ID [add your ID here]}
}{
\author{Jakob Thumm$^{1}$, Marian Frei$^{2}$, Tianle Ni$^{3}$, Matthias Althoff$^{4}$, and Marco Pavone$^{1}$%
\thanks{$^{1}$Jakob Thumm and Marco Pavone are with the Department of Aeronautics and Astronautics,
        Stanford University. {\tt\small thumm@stanford.edu, pavone@stanford.edu}}%
\thanks{$^{2}$Marian Frei is with the Chair of Imaging and Computer Vision,
        RWTH Aachen University. {\tt\small marian.frei@lfb.rwth-aachen.de}}%
\thanks{$^{3}$Tianle Ni is with the School of Artificial Intelligence,
        Shanghai Jiao Tong University. {\tt\small tianle.ni@tum.de}}%
\thanks{$^{4}$Matthias Althoff is with the Department of Computer Engineering,
        Technical University of Munich. {\tt\small althoff@tum.de}}%
}
}

\begin{document}

\maketitle
\thispagestyle{empty}
\pagestyle{empty}

\begin{abstract}
Safe human-robot collaboration (HRC) requires accurate human pose estimation and motion prediction to prevent critical collisions.
Existing certifiable safe HRC approaches are highly conservative or rely on marker-based motion tracking, while vision-based pose estimators lack the statistical guarantees required for certification in accordance with ISO 13849-1.
Hence, we propose a pipeline that predicts 3D human motion and strong probabilistic bounds on the prediction error using conformal prediction.
A gradient-based monitor detects out-of-distribution input poses and replaces them with poses from past predicted motions to maintain smooth operation.
The resulting conformal prediction sets directly integrate into the provably safe HRC approach SARA shield.
In experiments on the Human3.6M dataset and a real-world HRC setting, our conformal prediction sets have a \num{7.6} times smaller volume than model-based predictions, and we bound the probability of a dangerous failure per hour by \num{9.5e-7} with \SI{99.999}{\percent} confidence under our test distribution, which is necessary but not sufficient for performance level d.
\ifthenelse{\boolean{anon}}{
  All code and models are publicly available but are omitted here for anonymity.
}{
  All code and models are available at \url{https://jakob-thumm.com/conformal_human_motion_prediction/}.
}
\end{abstract}

\section{Introduction}\label{sec:introduction}
Autonomous robots will become an integral part of our society, performing tedious and dangerous tasks in industry, households, and healthcare.
To ensure safety in these human-centered environments, it is crucial to accurately estimate human poses, predict their motion, and control the robot to prevent critical collisions with the human.
For industrial applicability, these safety components have to be certified in accordance with ISO 13849-1:2023, which rates them by a performance level (PL) that ranges from PL~a to PL~e.
Each level corresponds to a range of the average probability of a dangerous failure per hour ($\mathrm{PFH_D}$).
The required level follows from the risk assessment of the application, and collaborative robot applications commonly target PL~d as specified in ISO 10218-1:2021, which spans $10^{-7} \le \mathrm{PFH_D} < 10^{-6}$.

Current certifiable safety approaches~\cite{haddadin_2012_MakingRobots, beckert_2017_OnlineVerification, thumm_2022_ProvablySafe, thumm_2026_GeneralSafety} are highly conservative as they assume worst-case human motion at all times and limit the admissible sensor setup.
Learning-based approaches enable pose estimation from portable sensors~\cite{li_2021_HumanPose, xu_2024_ViTPoseVision, khirodkar_2025_SapiensFoundation}, e.g., RGB-D cameras, and accurate motion prediction~\cite{mao_2019_LearningTrajectory, mao_2020_HistoryRepeats, ma_2022_ProgressivelyGenerating, saadatnejad_2024_ReliableHuman}.
However, these approaches lack strong statistically bounded guarantees with a $\mathrm{PFH_D} < 10^{-6}$ and they can catastrophically fail if inputs are out of distribution (OOD).

To alleviate these shortcomings, we propose a certifiable framework for vision-based safety in human-robot collaboration (HRC), which we visualize in~\cref{fig:overview}.
Given a stereo camera input, we estimate the two-dimensional (2D) human pose and its covariances in the two camera frames.
Then, we perform an uncertainty-aware triangulation to retrieve the three-dimensional (3D) pose with covariances.
Given a history of human poses, we predict future 3D poses and their covariances.
Based on the predicted covariances, we propose conformal prediction sets to over-approximate the uncertainty in the motion prediction with a $1-\epsilon$ confidence bound.
To detect critical OOD inputs in the pose estimation and motion prediction, we use the gradient-based OOD detection method in~\cite{miani_2024_SketchedLanczos}.
By reusing past motion predictions as potential replacements for OOD inputs, we establish a smooth operation of our pipeline in unseen situations.
The conformal prediction sets returned by our pipeline directly feed into the provably safe HRC approach SARA shield~\cite{thumm_2026_GeneralSafety}, which establishes certifiable safety guarantees at all times.

\begin{figure*}[t]
    \centering
    \includegraphics[width=\linewidth]{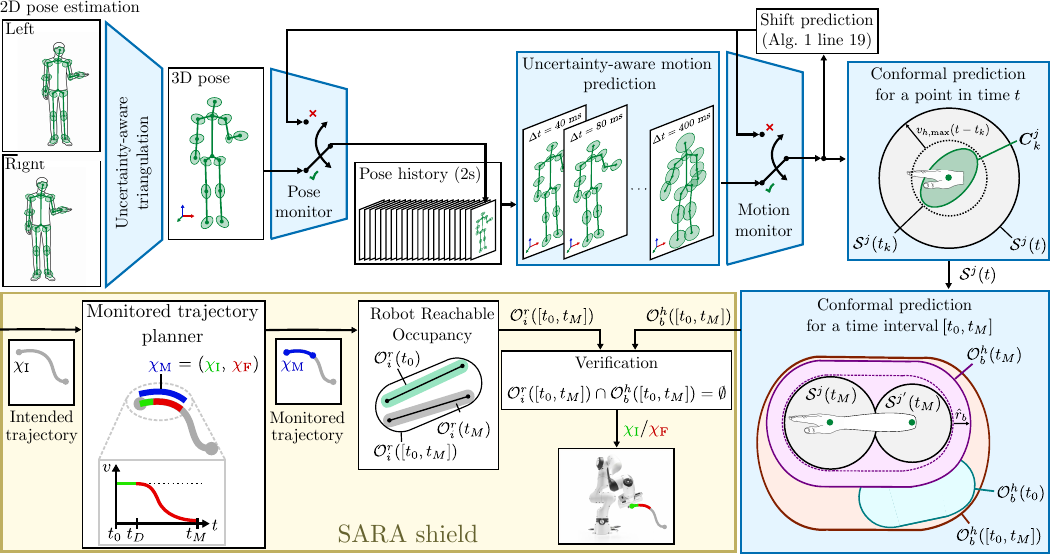}
    \caption{Overview of our safe HRC framework showing our novel contributions in blue. We estimate the human pose in the left and right camera image using YOLO26 and compute the 3D human pose and covariance matrices from triangulation. From the past $K_I$ poses, we predict the next $K_P$ 3D poses and their covariance matrices. If the pose or motion monitor detects that an input is OOD, we replace the predicted pose(s) with the pose(s) from the last motion prediction. We use conformal prediction bounds to determine the reachable occupancy of the human over the time horizon $[t_0, t_M]$ from the current control step $t_0$ until the robot stops at $t_M$, which we use to verify the safety of the robot trajectory.
	}
  \label{fig:overview}
\end{figure*}

The main contribution of our work is a framework for safe HRC from vision-based inputs featuring
\begin{itemize}
    \item 3D human motion prediction with pose uncertainty propagation,
    \item conformal prediction sets for human poses,
    \item a method to handle OOD inputs in a continuous prediction setting, and
    \item a scheme to estimate the $\mathrm{PFH_D}$ from limited motion data.
\end{itemize}
We evaluate our pipeline on the Human3.6M (H36M) dataset~\cite{ionescu_2013_Human36m} and in a real-world HRC setting with SARA shield~\cite{thumm_2026_GeneralSafety}. 
In our experiments, we find that (i) our conformal prediction sets have a \num{7.6} times smaller volume than model-based predictions following ISO 13855:2010, (ii) our system achieves a $\mathrm{PFH_D} \le \num{9.5e-7}$ with \SI{99.999}{\percent} confidence under a test distribution in which the human spends \SI{2.3}{\minute \per \hour} in the workspace of the robot, which is necessary but not sufficient for PL~d, and (iii) our proposed OOD handling reduces interruptions in real-world operation by \SI{24}{\percent}.

\section{Related Work}\label{sec:related_work}

Currently, human safety in HRC is guaranteed through the safety zone approaches of ISO 10218-1, which stop the robot if a human enters a pre-defined zone.
Safety zone approaches are highly conservative as they consider all objects in the safety zone as human and overapproximate the human occupancy with a cylinder representing the human arm span.
SARA shield and other state-of-the-art safe HRC approaches use articulated human poses to provide less conservative safety guarantees~\cite{haddadin_2012_MakingRobots, beckert_2017_OnlineVerification, thumm_2022_ProvablySafe, thumm_2026_GeneralSafety}.
These approaches use highly conservative motion models, e.g., that the human can move with up to $v_{h, \text{max}} = \SI{2.0}{\meter \per \second}$ in any direction as defined in ISO 13855.
Learning-based human motion prediction models~\cite{mao_2019_LearningTrajectory, mao_2020_HistoryRepeats, ma_2022_ProgressivelyGenerating, saadatnejad_2024_ReliableHuman} are less conservative but often lack heteroscedastic aleatoric uncertainty estimates and propagation of the pose estimation uncertainty.
Previous works that use learning-based models to estimate the human pose from more mobile vision-based sensors, e.g., \mbox{RGB-D} cameras, assume a fixed maximal estimation error~\cite{zanchettin_2016_SafetyHumanrobot, lacevic_2023_SafeHumanRobot} or provide uncertainty estimates without high statistical guarantees~\cite{li_2021_HumanPose, xu_2024_ViTPoseVision, khirodkar_2025_SapiensFoundation}.

Conformal prediction returns sets that contain the true output with a user-defined probability for any black-box predictor~\cite{vovk_2022_AlgorithmicLearning, angelopoulos_2023_ConformalPrediction} with instance-adaptive set sizes~\cite{papadopoulos_2008_NormalizedNonconformity, messoudi_2022_EllipsoidalConformal}.
In robotics, conformal prediction sets of learned motion prediction render model predictive controllers provably safe~\cite{lindemann_2023_SafePlanning, dixit_2023_AdaptiveConformal}, calibrate learned reachable tubes~\cite{muthali_2023_MultiAgentReachability}, and yield statistical reachable sets from data~\cite{tebjou_2023_DatadrivenReachability, devonport_2023_DataDrivenReachability}.
However, these approaches work in low-dimensional state spaces, e.g., a 2D pedestrian model~\cite[Sec. II]{lindemann_2023_SafePlanning}, and assume that this state is measured exactly.
They further calibrate at failure probabilities that prevent certification, e.g., \SI{5}{\percent} in~\cite{lindemann_2023_SafePlanning}.

For human pose and motion, conformal prediction calibrates 3D pose-shape estimates from video~\cite{zhang_2025_CUPSImproving} and calibrates 3D motion forecasts in a learned latent space~\cite{ma_2026_UncertaintyawareCalibrated}.
However, a guarantee that holds in a latent space does not transfer to the Euclidean space in which collisions occur, and Ma et al.~\cite{ma_2026_UncertaintyawareCalibrated} report an empirical pose-space coverage of \num{0.793} on H36M at a target coverage of \num{0.95}.

OOD detection methods identify inputs on which a learning-based model is unreliable~\cite{liu_2020_SimplePrincipled}, and recent work provide this signal without retraining the model~\cite{miani_2024_SketchedLanczos}.
However, approaches that react to OOD inputs~\cite{sinha_2024_RealTimeAnomaly, ganai_2025_RealTimeOutofDistribution} respond by stopping or replanning, which interrupts continuous operation.

\section{Preliminaries}\label{sec:preliminaries}
We first introduce the required notation and the problem of predicting sets that contain future human poses from camera images.
Subsequently, we present the set-based safety shield into which we integrate our predictions and the problem of obtaining certifiable safety guarantees for the overall vision-based system.

\subsection{Preliminaries on Conformal Prediction}\label{sec:prelim_conformal_predictions}
For brevity, we use the colon notation $A:B$ to refer to the sequence $A, A+1, \dots, B$.
We denote a point in the 2D image space as $(u, v) \in \mathbb{R}^2$ and a point in the 3D Euclidean space as $\pb \in \mathbb{R}^3$.
The human pose at timestep $k$ is denoted by a sequence of 3D points $\mathcal{P}_k = \pb^{1:J}_k \in \mathbb{R}^{J \times 3}$ with $J$ human joints, and the motion of a human over the discrete timesteps $1:K$ is denoted by a sequence of poses $\mathcal{P}_{1:K} \in \mathbb{R}^{K \times J \times 3}$.
We refer to the sequence of covariance matrices of a human pose at timestep $k$ as $\mathcal{C}_k = \Cb^{1:J}_k \in \mathbb{R}^{J \times 3 \times 3}$.
Finally, we refer to a sphere with center $\pb_c$ and radius $r$ as $\Bs(\pb_c, r)$.

Let $\mathcal{Z} = \{(\xb_i, \yb_i)\}_{i=1:N}$ be a set of i.i.d. training tuples with inputs $\xb_i \in \mathcal{X}$ and labels $\yb_i \in \mathcal{Y}$, which is divided into a training set $\mathcal{Z}^{\text{train}}$ and a calibration set $\mathcal{Z}^{\text{cal}}$ with $\mathcal{Z}^{\text{train}} \cap \mathcal{Z}^{\text{cal}} = \emptyset$.
Further, let a prediction model $f_\theta : \mathcal{X} \rightarrow \mathcal{Y}$ be trained on $\mathcal{Z}^{\text{train}}$. 
We denote a bounded set $\mathcal{S}_{\epsilon} \subseteq \mathcal{Y}$ as a conformal prediction set if $P(\yb \in \mathcal{S}_{\epsilon}) \ge 1-\epsilon$ with confidence level $1-\epsilon$~\cite[Sec. 2.3]{messoudi_2022_EllipsoidalConformal}.
For ease of notation, we drop the subscript $\epsilon$ of conformal prediction sets in this paper.
A non-conformity measure $A(\mathcal{Z}^{\text{train}}, \zb_i) \in \mathbb{R}$ is a function that scores how dissimilar a calibration sample $\zb_i$ is from the training dataset~\cite{vovk_2022_AlgorithmicLearning}, where we drop the input $\mathcal{Z}^{\text{train}}$ in the following text for brevity.

Given a sequence of past poses $\mathcal{P}_{(-K_I+1):0}$ and their covariances $\mathcal{C}_{(-K_I+1):0}$, human motion prediction aims to predict the sequence of future human 3D poses $\mathcal{P}_{1:K_P}$.
For the human poses with $J$ joints at $K_P$ future timesteps, we denote the sequence of $K_P$ sequences of $J$ conformal prediction sets as $\mathcal{S}^{1:J}_{1:K_P} = (\mathcal{S}^{1:J})_{1:K_P}$.

\begin{problembox}[label={problem1}]
Given a history of $K_I$ stereo image pairs $\{(\Ib_{k, 1}, \Ib_{k, 2})\}_{k=(-K_I+1):0}$, our goal is to predict the conformal prediction sets $\mathcal{S}^{1:J}_{1:K_P}$ that include the $K_P$ future human joint positions with $1-\epsilon$ confidence.
\end{problembox}
We assume that a single human is present throughout the observed and predicted horizon and leave the handling of humans entering and leaving the workspace to future work.

\subsection{Safe Human-Robot Collaboration}\label{sec:prelim_sara}
To ensure safety in HRC, we aim to integrate a vision-based human pose prediction into \shield{}~\cite{thumm_2022_ProvablySafe, thumm_2026_GeneralSafety}.
As visualized in~\cref{fig:overview}, at each control step, \shield{} constructs a monitored trajectory $\chix{M}$ that follows the intended trajectory for one timestep and appends a failsafe trajectory that brings the robot to a complete stop.
Without loss of generality, each planning cycle starts at time step $0$, so that the monitored trajectory spans the time interval $[t_0, \tb]$, where $\tb$ is the time at which the robot reaches a full stop.
\shield{} verifies the monitored trajectory as safe and executes it if the robot at time of collision has a safe kinetic energy; otherwise, the robot continues on the last verified trajectory.
We assume that the robot starts from a stopped state, so that safety is ensured indefinitely by induction~\cite[p.~5]{althoff_2019_EffortlessCreation}.

The verification of SARA shield~\cite{thumm_2026_GeneralSafety} is based on reachability analysis~\cite{althoff_2021_SetPropagation} of the robot and the human, where we replace the human reachable set computation with conformal prediction.
Since the exact reachable set $\Xs^e(t)$ of a system with bounded inputs and disturbances is generally not computable~\cite{gan_2018_ReachabilityAnalysis}, SARA shield uses over-approximative time-point reachable sets $\Xs(t) \supseteq \Xs^e(t)$ and the time-interval reachable sets ${\Xs([t_a, t_b]) = \bigcup_{t \in [t_a, t_b]} \Xs(t)}$.
The operator $\occ()$ returns the set of points in Euclidean space that a system can occupy for a given reachable set, which yields the reachable occupancy ${\mathcal{O}([t_a, t_b]) = \occ(\Xs([t_a, t_b]))}$.
The SaRA tool~\cite{schepp_2022_SaRATool}, which forms the basis of \shield{}, computes the reachable occupancies $\hum{\mathcal{O}}_b([t_0, \tb])$ of all human body parts $b = 1:B$, assuming a bounded speed of each body part, and the reachable occupancies $\Orobia([t_0, \tb])$ of all robot links $i = 1:N_L$ following the monitored trajectory.
As the monitored trajectory is known and the tracking error of the robot is bounded and small, the robot occupancies are tight.
In the simplest formulation of \shield{}, the monitored trajectory is verified as safe if no human body part can collide with any robot link
\begin{align}\label{eq:verification}
	c_{\text{safe}}(\chix{M}) \coloneqq \bigwedge_{i=1}^{N_L} \bigwedge_{b=1}^{B} \hum{\mathcal{O}}_b([t_0, \tb]) \cap \Orobia([t_0, \tb]) = \varnothing \, .
\end{align}
Thumm et al.~\cite{thumm_2026_GeneralSafety} loosen this constraint to allow contacts with humans while the kinetic energy of the robot is below pain and injury limits; however, we only consider \eqref{eq:verification} in this work.
To enable efficient intersection checks in \eqref{eq:verification}, \shield{} represents the reachable occupancies as capsules ${\mathcal{K}(\pb^1, \pb^2, r) = \overline{\pb^1 \pb^2} \oplus \Bs(\boldsymbol{0}, r)}$, where $\overline{\pb^1 \pb^2}$ is the line segment between the points $\pb^1$ and $\pb^2$, $\oplus$ is the Minkowski sum, and $\boldsymbol{0}$ is a vector of zeros~\cite{schepp_2022_SaRATool}.

\begin{problembox}[label={problem2}]
Develop a robot control system that ensures $\mathrm{PFH_D} < 10^{-6}$ with confidence $1-\epsilon_{\mathrm{D}}$ in a pre-defined HRC setup and human motion distribution.
\end{problembox}

\section{Conformal Human Motion Prediction}\label{sec:conformal_motion_prediction}

This section describes how we train a vision-based human motion prediction that provides strong probabilistic conformal guarantees to solve~\cref{problem1}.

\subsection{Conformal Prediction Sets}\label{sec:method_conformal_prediction}
For training accurate uncertainty estimates, it is beneficial to predict heteroscedastic aleatoric covariance matrices $\mathcal{C}_{1:K_P}$~\cite{russell_2021_MultivariateUncertainty}.
However, the high fidelity of ellipsoids is often not needed in formal verification, and set representations such as spheres and capsules are preferred as they allow for fast intersection checks~\cite{schepp_2022_SaRATool}.
Hence, we adapt the procedure in~\cite[Sec. 2.3]{messoudi_2022_EllipsoidalConformal} to construct spherical conformal prediction sets from covariance matrices.
Let $\hat{\pb}^j_{k}$ and $\Cb_{k}^j$ be the predicted mean and covariance matrix of the $j$-th joint at timestep $k$ and $\db^j_k = \pb^j_k-\hat{\pb}^j_k$ be the residual.
Following the split-conformal procedure~\cite[Sec. 2.3]{messoudi_2022_EllipsoidalConformal}, for every calibration tuple $\zb_i \in \mathcal{Z}^{\text{cal}}$, we define the per-joint, per-timestep non-conformity score as
\begin{equation}\label{eq:non_conformity_score}
	A_k^j(\zb_i) = \frac{\|\db_{k,i}^j\|_2}{\sqrt{\lambda_{\max}(\Cb_{k,i}^j)}} \, ,
\end{equation}
where $\lambda_{\max}(\Cb_{k,i}^j)$ is the maximal eigenvalue of $\Cb_{k,i}^j$\footnote{As covariance matrices are positive semi-definite, $\lambda_{\max}(\Cb_{k,i}^j) \in \Rbb^{+}$.}.
Let $\alpha_k^j$ be the $\big\lceil (1-\epsilon)(|\mathcal{Z}^{\text{cal}}|+1) \big\rceil$-th smallest score in $\left\{A_k^j(\zb_i) \, \mid \, \zb_i \in \mathcal{Z}^{\text{cal}}\right\}$~\cite[Sec. 2.3, Step 3 and 4]{messoudi_2022_EllipsoidalConformal}.
\begin{proposition}\label{prop:sphere_set}
	Given a prediction $\hat{\pb}^j_{k, i}$ and covariance $\Cb_{k,i}^j$, the sphere ${\mathcal{S}_{k,i}^j = \mathcal{B}\left(\hat{\pb}^j_{k, i}, \alpha_k^j \sqrt{\lambda_{\max}(\Cb_{k,i}^j)}\right)}$ 
	is a conformal prediction set.
\end{proposition}
\begin{proof}
	The proof is a special case of~\cite[Theo. 1]{messoudi_2022_EllipsoidalConformal}.
\end{proof}
The guarantee of \cref{prop:sphere_set} requires the test tuple to be exchangeable with the calibration tuples.

\subsection{Uncertainty-Aware 3D Motion Prediction}\label{sec:method_human_motion}
For our human motion prediction model depicted in~\cref{fig:motion_model_architecture}, we extend the discrete cosine transform (DCT) transformer model of~\cite{mao_2020_HistoryRepeats}.
In essence, we add uncertainty embeddings to handle uncertain inputs and an uncertainty decoder to predict the future pose uncertainties.
To predict $\mathcal{C}_{1:K_P}$, we follow the Cholesky factorization approach in~\cite{russell_2021_MultivariateUncertainty}, which yields stable training and avoids invalid covariance matrices.

\begin{figure}[t]
  \centering
  \includegraphics[width=\columnwidth]{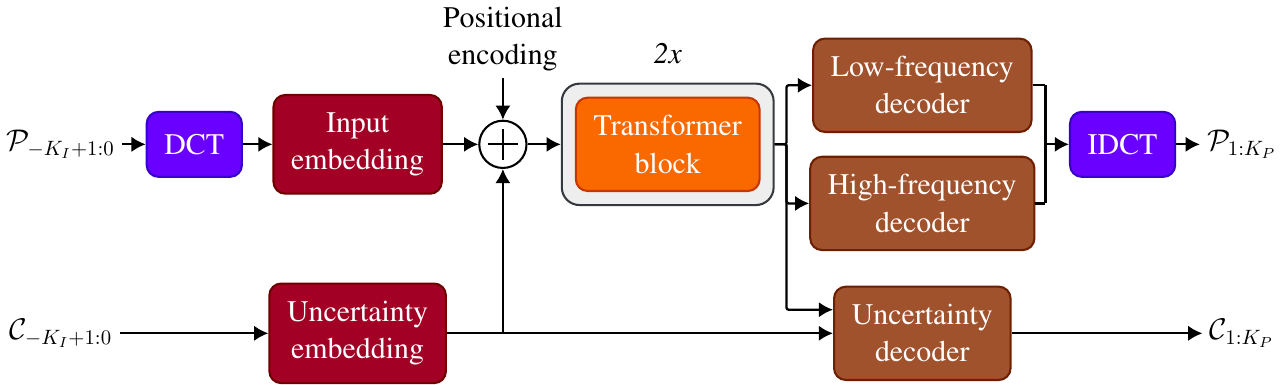}
  \caption{Overview of the motion model architecture.}
  \label{fig:motion_model_architecture}
\end{figure}

\subsubsection{Training Objective}\label{sec:training_objective}
In addition to the typically used multivariate Gaussian negative log-likelihood (NLL) loss~\cite{russell_2021_MultivariateUncertainty}, we employ three loss terms to improve the generalization of the uncertainty prediction in rare edge cases.

First, high input uncertainties, e.g., during occlusions, are rare but critical, so we use importance weighting to upsample these cases.
Let $r^j_{\text{in}} = \sqrt{\chi^2_3(1-\epsilon)\,\lambda_{\max}(\Cb^j_{0})}$ be the reported input-uncertainty radius of joint $j$, obtained from the last input covariance $\Cb^j_{0}$, where $\chi^2_3(1-\epsilon)$ is the $(1-\epsilon)$-quantile of the chi-square distribution with three degrees of freedom.
We define the importance weight as
\begin{equation}
w^j = \min\!\Big(\max\big((r^j_{\text{in}}/\tilde{r}_{\text{in}})^{\gamma},\,1\big),\,w_{\max}\Big),
\end{equation}
where $\tilde{r}_{\text{in}}$ is the median input radius over the batch, $\gamma \in \{0, 1\}$, and $w_{\max}$ caps the weight.
The corresponding weighted NLL loss~\cite{russell_2021_MultivariateUncertainty} is
\begin{equation}
\mathcal{L}_{\text{wNLL}} =
\frac{1}{W}\sum_{k=1}^{K_P}\sum_{j=1}^{J} w^j\!\left(
\tfrac{1}{2}\log|\Cb^j_k|
+ \tfrac{1}{2}{\db^j_k}^\top{\Cb^j_k}^{-1}\db^j_k\right),
\end{equation}
with $W = K_P\sum_{j} w^j$.

Second, as our target metric is the coverage of the conformal prediction sets, we add a pinball loss~\cite{steinwart_2011_EstimatingConditional} to bring their radius toward the $(1-\epsilon)$-quantile of the residual norm.
For training, we assume that the residuals are normally distributed and define $\hat{r}^j_k = \sqrt{\chi^2_3(1-\epsilon)\,\lambda_{\max}(\Cb^j_k)}$ as the Gaussian set radius of the estimated covariance matrix of joint $j$ at time $k$.
We define the pinball loss as
\begin{equation}
\mathcal{L}_{\text{pin}} =
\frac{1}{W}\sum_{k=1}^{K_P}\sum_{j=1}^{J} w^j\,
\rho_\kappa\!\big(\|\db^j_k\|_2 - \hat{r}^j_k\big),
\end{equation}
with $\rho_\kappa(u)=\max(\kappa u,(\kappa-1)u)$, where we set $\kappa = 1-\epsilon$.
Training with $\mathcal{L}_{\text{pin}}$ pushes the score thresholds toward $\lim_{\mathcal{L}_{\text{pin}}\to 0}\alpha^j_k = \sqrt{\chi^2_3(1-\epsilon)}$.
We ensure that $\mathcal{L}_{\text{pin}}$ only shapes the covariances and not the predicted positions by stopping the gradient of $\mathcal{L}_{\text{pin}}$ on the residual $\db^j_k$.

Finally, to stabilize the mean prediction, we add the $\ell_2$ pose loss $\mathcal{L}_{\text{pose}} = \frac{1}{J K_P}\sum_{k}\sum_{j}\|\db^j_k\|_2^2$, resulting in
\begin{equation}
\mathcal{L}_{\text{total}} =
\mathcal{L}_{\text{wNLL}} + 
\mathcal{L}_{\text{pose}} +
\mathcal{L}_{\text{pin}} \, .
\end{equation}

\subsubsection{Stable Multi-Phase Optimization}\label{sec:multi-phase_opt}
Direct end-to-end training with covariance prediction can be unstable. 
Therefore, we follow a staged procedure.
First, we train the backbone and pose prediction head without uncertainty head using $\mathcal{L}_{\text{pose}}$ only.
Subsequently, we add the uncertainty head and train it decoupled from the position head with the backbone frozen, where we set $\gamma = 0$.
Finally, we fine-tune the full model end-to-end with activated input uncertainty weighting $\gamma = 1$.
Throughout training, we use a cosine learning rate schedule that resets at the beginning of each stage.

\subsection{Conformal Calibration}\label{sec:conformal_prediction}
The analytic Gaussian value $\alpha_k^j = \sqrt{\chi^2_3(1-\epsilon)}$ used during training would be exact only if the residuals were truly Gaussian with the predicted covariance $\Cb_k^j$, which does not hold for a learned, potentially misspecified covariance estimate.
Hence, we calibrate $\alpha_k^j$ after training on a held-out calibration set $\mathcal{Z}^{\text{cal}}$ as described in~\cref{sec:method_conformal_prediction}.
As an ablation, we replace the $J \cdot K_P$ per-joint, per-timestep thresholds by the single threshold $\alpha_{\max}$, obtained as the $\big\lceil (1-\epsilon)(|\mathcal{Z}^{\text{cal}}|+1) \big\rceil$-th smallest value of the per-sample maximum score $A_{\max}(\zb_i) = \max_{j, k} A_k^j(\zb_i)$, so that the resulting sets hold simultaneously for an entire motion prediction with probability $\ge 1-\epsilon$, rather than marginally for each joint and timestep.

\section{OOD Events}\label{sec:pipeline}
This section describes how we handle OOD inputs in our pipeline depicted in~\cref{fig:overview}.

\subsection{OOD Detection}\label{sec:ood_detection}
The exchangeability between the test tuple and the calibration tuples required by \cref{prop:sphere_set} is unlikely to hold in unseen situations.
In these cases, the residual $\|\db^j_{k,i}\|_2$ tends to grow faster than the predicted uncertainty $\sqrt{\lambda_{\max}(\Cb^j_{k,i})}$ and inflate the non-conformity score beyond $\alpha_k^j$.
Therefore, we deploy the sketched Lanczos uncertainty (SLU) OOD detection method of~\cite{miani_2024_SketchedLanczos} to detect OOD inputs in the pose estimation $\mathrm{SLU}_{\text{2D}}$ and motion prediction $\mathrm{SLU}_{\text{mot}}$.
We calibrate the two OOD thresholds $\tau_{\text{2D}}$ and $\tau_{\text{mot}}$ on the calibration datasets, so that $P(\text{SLU}(\zb_i) \le \tau_{\text{OOD}}) \ge 1-\epsilon_{\text{OOD}}$.

Since the SLU computation time scales linearly with the number of output parameters, we add a randomly initialized linear layer to each model, which acts as a random projection into an \num{8}-dimensional subspace.
This is a trade-off between OOD detection accuracy and inference time.
Lastly, following the single-human assumption of~\cref{problem1}, we treat a missing human in the frame as OOD.

\subsection{Handling OOD Events}\label{sec:handling_ood}
\begin{algorithm}[t]
	\caption{Uncertainty-Aware Human Pose Pipeline}\label{alg:pipeline}
	\begin{algorithmic}[1]
		\Require Stereo camera stream $\{(\Ib_{k, 1}, \Ib_{k, 2})\}_{k \geq 1}$, $f_{\text{2D}}$, $f_{\text{mot}}$, $\mathrm{SLU}_{\text{2D}}$, $\mathrm{SLU}_{\text{mot}}$, $\tau_{\text{2D}}$, $\tau_{\text{mot}}$, $K_I$, $K_P$, $N_{\text{req}}$
		\State Initialize pose buffer $\mathcal{H} \leftarrow \varnothing$, pose validity buffer $\vb \leftarrow \boldsymbol{0}_{K_I}$, and motion buffer $\mathcal{M}$ to NaN \label{alg:line:init}
		\While{True}
		\State Acquire stereo frames $(\Ib_{k,1}, \Ib_{k,2})$
		\State Estimate 2D pose and covariances \Comment{Sec.~\ref{sec:method_human_pose}}
		\State Compute 3D pose $\mathcal{P}_k$ and covariances $\mathcal{C}_k$ via stereo triangulation \Comment{Sec.~\ref{sec:method_human_pose}}
		\State Compute $s_{\text{2D}} \leftarrow \mathrm{SLU}_{\text{2D}}(\Ib_{k,1})$ \label{alg:line:slu2d}
		\If{$s_{\text{2D}} \le \tau_{\text{2D}}$} \label{alg:line:pose:if}
		\State Append $(\mathcal{P}_k,\, \mathcal{C}_k)$ to $\mathcal{H}$; set $v_k \leftarrow 1$ \label{alg:line:pose:valid}
		\Else
		\State Append $\mathcal{M}[0]$ to $\mathcal{H}$; set $v_k \leftarrow 0$ \label{alg:line:pose:ood}
		\EndIf
		\If{$|\mathcal{H}| < K_I$} \textbf{continue} \EndIf \Comment{Wait until pose buffer is full} \label{alg:line:bufguard}
		\State Predict $(\mathcal{P}_{1:K_P},\, \mathcal{C}_{1:K_P}) \leftarrow f_{\text{mot}}(\mathcal{H})$ \Comment{Sec.~\ref{sec:method_human_motion}} \label{alg:line:motpredict}
		\State Compute $s_{\text{mot}} \leftarrow \mathrm{SLU}_{\text{mot}}(\mathcal{H})$ \label{alg:line:slumot}
		\If{$s_{\text{mot}} \le \tau_{\text{mot}}$ \textbf{and} $\sum_{i=K_I - N_{\text{req}}+1}^{K_I} v_i = N_{\text{req}}$} \label{alg:line:mot:if}
		\State $\mathcal{M} \leftarrow (\mathcal{P}_{1:K_P},\, \mathcal{C}_{1:K_P})$ \Comment{If input is in-distribution and last $N_{\text{req}}$ poses came from images, accept new prediction} \label{alg:line:mot:accept}
		\Else
		\State $\mathcal{M}[0:(K_P-1)] \leftarrow \mathcal{M}[1:K_P]$; $\mathcal{M}[K_P] \leftarrow \text{NaN}$ \Comment{Continue on prior prediction} \label{alg:line:mot:shift}
		\EndIf
		\State $\mathcal{O}_{1:K_P} \leftarrow \mathrm{ConformalSets}(\mathcal{M})$ \Comment{Sec.~\ref{sec:conformal_prediction}} \label{alg:line:conformal}
		\State Publish reachable occupancies $\mathcal{O}_{1:K_P}$ \label{alg:line:publish}
		\EndWhile
	\end{algorithmic}
\end{algorithm}
Assuming a fixed camera frequency $f_{\text{cam}} = \tfrac{1}{t_{k+1}-t_{k}}$, our pipeline would not return any prediction in the timeframe $K_I / f_{\text{cam}}$ after an invalid 2D pose as our motion prediction requires $K_I$ valid input poses.
With, e.g., $K_I=50$ and $\epsilon_{\text{OOD}} = \SI{5}{\percent}$, this would induce a failed motion prediction in \SI{92.3}{\percent} of timesteps.
To circumvent this, we propose an algorithm that reuses previous motion predictions for OOD pose estimation.

We describe the total process of our human pose pipeline in~\cref{alg:pipeline}.
At each timestep, $\mathrm{SLU}_{\text{2D}}$ scores the current camera image against the training distribution of $f_{\text{2D}}$ (line~\ref{alg:line:slu2d}).
If the 2D OOD score of the first image is below the calibrated threshold $\tau_{\text{2D}}$, the estimated pose $\mathcal{P}_k$ and covariance $\mathcal{C}_k$ are appended to the pose buffer $\mathcal{H}$ and the corresponding validity flag is set to $v_k = 1$ (line~\ref{alg:line:pose:valid}).
If the pose is classified as OOD, we discard the estimated 3D pose and replace it by the first entry of the current motion prediction buffer $\mathcal{M}[0]$ and set the validity flag to $v_k = 0$ (line~\ref{alg:line:pose:ood}).
The resulting pose buffer contains a temporally consistent sequence of plausible poses during periods where the pose estimator is unreliable.

Once the pose buffer contains $K_I$ entries (line~\ref{alg:line:bufguard}), we start predicting human motions using $f_{\text{mot}}$ (line~\ref{alg:line:motpredict}) and the motion OOD score with $\mathrm{SLU}_{\text{mot}}$ (line~\ref{alg:line:slumot}).
The motion buffer $\mathcal{M}$ is updated with the new prediction if (i) the motion input is in-distribution ($s_{\text{mot}} \le \tau_{\text{mot}}$), and (ii) the last $N_{\text{req}}$ entries of the pose buffer all originate from valid image observations (lines~\ref{alg:line:mot:if}--\ref{alg:line:mot:accept}).
The second condition prevents a continuous feedback loop in the motion prediction if the 2D pose input is OOD.
If either condition fails, the motion buffer is shifted by one timestep to the left and the vacated slot at the end of the buffer is marked as invalid (line~\ref{alg:line:mot:shift}).
Normal operation is resumed after $N_{\text{req}}$ consecutive valid image-based poses have been observed.

Reusing a previous prediction affects safety, since the conformal coverage guarantee of~\cref{sec:conformal_prediction} is not calibrated on reused predictions.
However, a reused prediction is more likely to be exchangeable with the calibration set than an OOD input.
We test the influence of OOD inputs on the prediction accuracy experimentally in~\cref{sec:full_pipeline_eval}. 

\section{Obtaining Certifiable Guarantees}\label{sec:cert_guarantees}
In this section, we highlight how our method can be integrated in a formal verification framework, such as SARA shield~\cite{thumm_2026_GeneralSafety}.
Next to the structural integration, we discuss how $\mathrm{PFH_D}$ can be bounded, which is a requirement to certification in accordance with ISO~13849-1.

\subsection{Integration in SARA Shield}\label{sec:sara-shield}
To verify the safety constraint in~\cref{eq:verification}, we compute the reachable occupancies of all human body parts over a given time horizon $[t_0, \tb]$.
At every control step, we obtain $[t_0, \tb]$ from \shield{}.
Our conformal predictions only provide sets at predefined timesteps $t_{1:K_P}$, which seldom contain $\tb$ exactly.
To retrieve the conformal prediction sets at any time $t \ge t_{1}$, we find the maximum time $t_k$ for all $k=1:K_P$ for which $t_k \le t$, set $\Delta t= t - t_k$, and obtain $\mathcal{S}^j(t)$ by extending the radius of $\mathcal{S}^j_k$ by $\Delta r = \Delta t v_{h, \text{max}}$ with $v_{h, \text{max}} = \SI{2.0}{\meter \per \second}$ as defined in ISO 13855.

To obtain the time-point occupancy of a human body part, we apply the occupancy operator to the conformal prediction sets of the two joints $j$ and $j'$ adjacent to body part $b$ with $\hum{\mathcal{O}}_b(t) = \occ(\mathcal{S}^{j}(t) \cup \mathcal{S}^{j'}(t)) = \mathcal{K}\big(\hat{\pb}^{j}, \hat{\pb}^{j'}, \max(r^{j}, r^{j'}) + \hat{r}_b \big)$, where $\hat{\pb}^{j}$ and $r^{j}$ are the center and radius of $\mathcal{S}^{j}(t)$ and we set $\hat{r}_b$ to the measurements given in DIN 33402-2:2020. %
We obtain the time-interval occupancy $\hum{\mathcal{O}}_b([t_0, \tb])$ with SaRA~\cite{schepp_2022_SaRATool} and use it in place of the human occupancies in~\cref{eq:verification}.

\subsection{Probability of Dangerous Failures Per Hour}\label{sec:prob_cert}
A fault in our pose pipeline does not by itself create a dangerous failure, as defined in~\cref{sec:prelim_sara}, since a dangerous event additionally requires the robot to come into contact with the human.
Hence, we define the two events:
\begin{itemize}
  \item[$F$:] \emph{prediction failure}: $\exists j, k: \pb^j_k \notin \mathcal{S}^j_k$,
  \item[$D$:] \emph{dangerous failure}: a contact occurs and $c_{\text{safe}}(\chix{M}) = 1$.
\end{itemize}
Since the conformal prediction sets over-approximate the true human occupancy in the verified time interval, an empty intersection in~\eqref{eq:verification} excludes a contact if there is no prediction failure, so $D \Rightarrow F$ and $P(D) = P(D, F) = P(F)\,P(D \mid F)$.
With an operating hour containing $N_h = \SI{3600}{\second} / t_{\text{cycle}}$ verification cycles, $\mathrm{PFH_D} = N_h P(F) P(D \mid F)$.

Our core idea is to determine $P(F)$ on the test set $\mathcal{Z}^{\text{test}}$ and then randomly place the failed motion predictions in varying positions and orientations around the robot to simulate a typical HRC situation.
As the test motions and the simulated human placements are drawn independently, both confidence bounds hold simultaneously with $1 - \epsilon_{\mathrm{D}} = (1-\epsilon_F)(1-\epsilon_{D \mid F})$, and we allocate the budget symmetrically, $\epsilon_F = \epsilon_{D \mid F} = 1 - \sqrt{1-\epsilon_{\mathrm{D}}}$.
Let $\mathrm{B}^{-1}(q; a, b)$ denote the $q$-quantile of the Beta distribution with shape parameters $a$ and $b$, which yields the one-sided Clopper-Pearson upper limit $\mathrm{B}^{-1}(1-\epsilon; k+1, n-k)$ on a binomial proportion after observing $k$ events in $n$ trials at confidence $1-\epsilon$.

We evaluate $F$ on all $|\mathcal{Z}^{\text{test}}|$ prediction windows and observe $k_F$ failures, which gives $P(F) \approx k_F / |\mathcal{Z}^{\text{test}}|$.
Consecutive windows share $K_I-1$ of their $K_I$ input poses and therefore re-observe almost the same motion.
Hence, we compute the confidence interval on $P(F)$ over an effective sample size $N_{\text{eff}} = |\mathcal{Z}^{\text{test}}| / (2 L_{\text{corr}})$ rather than $|\mathcal{Z}^{\text{test}}|$, where $L_{\text{corr}} = 1 + 2\sum_{l \ge 1} \rho_l$ is the integrated autocorrelation time of the failure indicator and the factor $2$ accounts for the H36M dataset containing both phase offsets of the \SI{50}{\hertz} recordings. 
Then $P(F) \le \mathrm{B}^{-1}(1-\epsilon_F; k_{\text{eff}}+1, N_{\text{eff}}-k_{\text{eff}})$ with $k_\text{eff} = N_\text{eff} \, k_F / |\mathcal{Z}^{\text{test}}|$.

We determine $P(D \mid F)$ by replaying each of the $k_F$ failed predictions at random human positions and orientations around the robot base and at a random phase of a typical long-horizon robot trajectory.
Unlike the motion windows, these $N_D$ placements are drawn i.i.d. from a distribution of our choosing, so $P(D \mid F) \le \mathrm{B}^{-1}\!\left(1-\epsilon_{D \mid F};\ k_D+1,\ N_D-k_D\right)$ requires no independence assumption on the data. 
The chosen distribution should mirror the expected real-world HRC situation, where a high number of expected close interactions can be achieved by frequently sampling motions close to the robot. 
Combining the two bounds, we obtain with confidence $1-\epsilon_{\mathrm{D}}$
\begin{equation}\label{eq:pfh_bound}
\begin{split}
  \mathrm{PFH_D} \le N_h \,
  &\mathrm{B}^{-1}\!\left(1-\epsilon_F; k_{\text{eff}}+1, N_{\text{eff}}-k_{\text{eff}}\right) \\
  &\cdot \mathrm{B}^{-1}\!\left(1-\epsilon_{D \mid F};\, k_D{+}1,\, N_D{-}k_D\right) .
\end{split}
\end{equation}

A bounded $\mathrm{PFH_D}$ is necessary but not sufficient for a certain PL.
ISO~13849-1 additionally requires a control architecture of a defined category, together with sufficient diagnostic coverage, protection against common-cause failures, and measures against systematic and software faults.
Our pipeline has a single channel and provides no hardware redundancy, so it cannot on its own satisfy a fault-tolerant category such as Category~3.

\section{Experiments}\label{sec:experiments}
In this section, we evaluate the performance of our pose pipeline on real-world data and real-world robot deployment.
We mainly evaluate our method on the H36M benchmark~\cite{ionescu_2013_Human36m} in a dual-camera setup.
Following the standard protocol used in prior work~\cite{mao_2019_LearningTrajectory, guo_2023_BackMLP}, we use subjects S1, S6, S7, S8, and S9 for training, S5 for testing, and split S11 in half for validation and calibration.
Additionally, we test our approach in a real-world HRC setting
on a Franka Emika robot and an Intel RealSense 435i RGB-D camera.
We report all test results using $K_I = \num{50}$ frames, $K_P = \num{10}$ frames, $J=\num{13}$ joints, $f_{\text{cam}} = \SI{25}{\hertz}$, $\epsilon_{\text{OOD}} \approx \SI{10}{\percent}$, and calibration confidence of $1-\epsilon = \SI{99.99}{\percent}$.

We test three hypotheses:
\begin{itemize}
    \item[\textbf{H1}] Our conformal prediction sets are less conservative than the ISO~13855 model.
    \item[\textbf{H2}] Under our simulated test distribution our system achieves $\mathrm{PFH_D} < 10^{-6}$.
    \item[\textbf{H3}] Our OOD handling reduces invalid motion predictions at a small accuracy cost.
\end{itemize}

\subsection{Uncertainty-Aware 3D Pose Estimation}\label{sec:method_human_pose}
For our 2D pose estimation, we first adapt YOLO26~\cite{sapkota_2026_YOLO26Key} to return the covariance matrix $\Cb^j_{i,\text{2D}} \in \mathbb{R}^{2 \times 2}$ in addition to the 2D mean $(u_i^j, v_i^j)$ of each human joint $j = 1:J$.
We extract the approximated covariance matrix from the scale of the Gaussian distribution returned by the regression model of the flow-based regression~\cite{li_2021_HumanPose} used in YOLO26.

Our 3D depth perception differs between the H36M and real-world setup.
In H36M, we obtain two 2D poses from external RGB cameras and estimate the 3D joint positions by standard linear triangulation~\cite{hartley_2003_MultipleView} and estimate the cross-covariance from residual correlations over the calibration set $\mathcal{Z}^{\text{cal}}$~\cite{bryan_2012_ClusteringSynchronizing}.
Subsequently, we obtain $\mathcal{C}_{1:K_P}$ through first-order covariance propagation~\cite{lewandowski_2010_CovariancebasedVectornetworkanalyzer}. %
In the real world, we use an RGB-D camera, where we retrieve the depth information directly from its output and mark all 3D poses whose depth differs by more than \SI{0.8}{\meter} from the median depth as OOD.

\subsection{Conformal Prediction Evaluation}\label{sec:exp_conformal_prediction_set}

\begin{table*}[t]
    \centering
    \caption{Motion prediction results on H36M test data and PFH$_\text{D}$ evaluation in $N_D = \num{2.65e11}$ random placement experiments.}
    \label{tab:all_conformal_results}
    \begin{tabular}{lcccccccc}
        \toprule
        \multirow{2}[3]{*}{{Method}} & \multicolumn{2}{c}{Coverage} & \multicolumn{3}{c}{$\downarrow$ Volume (\unit{\cubic\meter})} & \multirow{2}[3]{*}{$\downarrow$ $P(F)$} & \multirow{2}[3]{*}{$\downarrow$ $k_D$} & \multirow{2}[3]{*}{$\downarrow$ PFH$_\text{D}$ (1/h)} \\
        \cmidrule(lr){2-3} \cmidrule(lr){4-6}
         & $\downarrow$ Miss rate & $\uparrow$ Nines of reliability & 5\% & 50\% & 95\% & & & \\
        \midrule
        ISO 13855 & \num{5.0e-6} & 5.30 & 0.017 & 0.686 & 3.244 & \num{2.87e-3} & \textbf{0} & \num{1.62e-7} \\
        Ours ($\alpha_{\max}$) & $\mathbf{1.0 \times 10^{-6}}$ & \textbf{6.00} & 0.072 & 0.413 & 6.352 & $\mathbf{4.62 \times 10^{-4}}$ & \textbf{0} & $\mathbf{1.91 \times 10^{-8}}$ \\
        Ours with OOD inputs & \num{1.6e-4} & 3.79 & 0.015 & 0.090 & 0.654 & \num{1.30e-2} & 4 & \num{9.48e-7} \\
        Ours OOD filtered & \num{1.6e-4} & 3.79 & \textbf{0.015} & \textbf{0.090} & \textbf{0.651} & \num{1.30e-2} & 4 & \num{9.50e-7} \\
        \midrule
        No conformal prediction $\left(\alpha_{k, \chi}^j\right)$ & \num{6.8e-3} & 2.17 & {0.002} & {0.011} & {0.174} & \num{1.80e-1} & \num{24722} & \textcolor{red}{\num{1.61e-2}} \\
        \bottomrule
    \end{tabular}
\end{table*}

We compare the validity of our conformal prediction sets against the SaRA tool~\cite{schepp_2022_SaRATool} with the constant velocity model of ISO 13855, where we calibrate the conformal prediction sets of the last estimated pose $\mathcal{S}^j_0$ to a confidence of $1-\epsilon = \SI{99.99}{\percent}$ and extend their radius by $v_{h, \text{max}} (t-t_0)$ with $v_{h, \text{max}} = \SI{2.0}{\meter \per \second}$.
We ablate our per-joint per-timestep conformal prediction scores against the analytic Gaussian values $\alpha_{k, \chi}^j = \sqrt{\chi^2_3(1-\epsilon)}$ and the single prediction threshold $\alpha_{\max}$.
Additionally, we evaluate our method with and without detected OOD instances filtered out.

We perform two evaluations whose results are reported in~\cref{tab:all_conformal_results}.
First, we test the coverage of our conformal prediction set on the H36M test data and report the percentage of ground-truth joint positions within the predicted sets and the \num{5}, \num{50}, and \num{95} percentiles of the predicted set volumes.
For the probabilistic certification in~\cref{sec:prob_cert}, we perform a simulation of typical HRC tasks, which uses the H36M test data for $\mathcal{Z}^\text{test}$ and a pre-recorded pick-and-place trajectory for the robot motion.
To test for high PLs, we evaluate $N_D = \num{2.65e11}$ random placements with $t_{\text{cycle}} = \SI{4}{\milli \second}$.
We set $r_{\text{test}} = \SI{10}{\meter}$, which emulates a typical HRC scenario, where a human is always present within a \SI{10}{\meter} radius of the robot and is in the path of the robot in \SI{2.3}{\minute} per operating hour.
Since ISO~13855 assumes that no human body part moves faster than $v_{h, \text{max}}$, we filter out all test data that violate this assumption in both experiments.

Our conformal prediction sets achieve a \num{7.6} times smaller median volume than the sets from the ISO~13855 model, confirming \textbf{H1}. 
Our miss rate of $p_{\text{miss}} = \num{1.6e-4}$, i.e., the rate at which the prediction lies outside of the predicted set, equals a nines-of-reliability score~\cite[Sec. 4.2]{hendrycks_2024_IntroductionAI} of $k = -\log_{10}(p_{\text{miss}}) = \num{3.79}$. 
Our evaluation shows that \textbf{H2} holds, as our method causes only \num{4} critical failures, achieving a PFH$_\text{D}$ of \num{9.50e-7} at \SI{99.999}{\percent} confidence.
Our statistical test also confirms the validity of using the SaRA tool with the constant velocity model of ISO 13855 on top of the conformal prediction sets from the learned pose estimation and of our conformal calibration with $\alpha_{\max}$, showing a trade-off between PFH$_\text{D}$ scores and set volumes.
Only the ablation with analytic Gaussian values $\alpha_{k, \chi}^j$ fails to achieve a $\mathrm{PFH_D} < 10^{-6}$, highlighting the importance of conformal prediction.
The OOD filtering removed \num{2} of the \num{436} prediction failures, which did not have an impact in this evaluation but could be crucial in other settings.

\subsection{Prediction Accuracy}\label{sec:exp_motion_prediction_acc}
For estimating the prediction accuracy of our model, we compare the mean per-joint position error (MPJPE)~\cite{mao_2019_LearningTrajectory} against the reported numbers in the common baselines HisRep~\cite{mao_2020_HistoryRepeats}, ST-DGCN~\cite{ma_2022_ProgressivelyGenerating}, ST-Trans~\cite{saadatnejad_2024_ReliableHuman}, and SiMLPe~\cite{guo_2023_BackMLP} in~\cref{tab:prediction-benchmark}.
We evaluate the model performance on ground-truth 3D pose inputs.
Our model achieves worse prediction results than the state of the art, which is likely due to only predicting \num{13} joints compared to the typically reported numbers on \num{22} joints.

\begin{table}[t]
    \centering
    \caption{Motion prediction test results on H36M.}
    \label{tab:prediction-benchmark}
    \begin{tabular}{lcccc}
        \toprule
        & \multicolumn{4}{c}{$\downarrow$ \textbf{MPJPE (mm)}} \\
        \cmidrule(lr){2-5}
        \textbf{Method} & \textbf{\SI{80}{\milli\second}} & \textbf{\SI{160}{\milli\second}} & \textbf{\SI{320}{\milli\second}} & \textbf{\SI{400}{\milli\second}} \\
        \midrule
        Repeating Last-Frame~\cite{guo_2023_BackMLP} & 23.8 & 44.4 & 76.1 & 88.2 \\
        One FC~\cite{guo_2023_BackMLP} & 14.0 & 33.2 & 68.0 & 81.5 \\
        HisRep~\cite{mao_2020_HistoryRepeats} & 10.4 & 22.6 & 47.1 & 58.3 \\
        ST-DGCN~\cite{ma_2022_ProgressivelyGenerating} & 10.3 & 22.7 & 47.4 & 58.5 \\
        ST-Trans~\cite{saadatnejad_2024_ReliableHuman} & 10.4 & 23.4 & 48.4 & 59.2 \\
        SiMLPe~\cite{guo_2023_BackMLP} & \textbf{9.6} & \textbf{21.7} & \textbf{46.3} & \textbf{57.3} \\
        \midrule
        Ours (stage 1) & 16.5 & 23.1 & 48.4 & 62.4 \\
        Ours (final) & 22.0 & 30.7 & 56.1 & 69.4 \\
        \bottomrule
    \end{tabular}
\end{table}

\subsection{Full Pipeline Evaluation}\label{sec:full_pipeline_eval}
We evaluate the efficacy of our OOD handling mechanism described in~\cref{alg:pipeline} by executing our full pipeline on the H36M test data with varying $N_{\text{req}}$ values.
Here, $N_{\text{req}} = K_I = 50$ indicates that any OOD input in the 2D pose estimation would result in an invalid motion prediction, and we use $N_{\text{req}} = 3$ in our real-world experiments.
Our results in~\cref{tab:full_pipeline_results} show that our OOD pipeline reduces the rate of invalid pose buffers $\sum_{i=K_I - N_{\text{req}}+1}^{K_I} v_i < N_{\text{req}}$ by \SI{24}{\percent} while only increasing the average MPJPE by \SI{1.9}{\percent}, supporting \textbf{H3}.
The total median prediction time on a PC with an NVIDIA GeForce RTX 5090 and AMD Ryzen 9 9950X is \SI{65.86}{\milli \second} dominated by $f_\text{2D}$ (\SI{13.20}{\milli \second}), $\mathrm{SLU}_\text{2D}$ (\SI{36.39}{\milli \second}), $f_\text{mot}$ (\SI{4.69}{\milli \second}), and $\mathrm{SLU}_\text{mot}$ (\SI{9.08}{\milli \second}).
We report an AUROC of \num{0.9976} for $\mathrm{SLU}_\text{2D}$ compared to the tiger-pose dataset and \num{0.9858} for $\mathrm{SLU}_\text{mot}$ compared to temporally shuffled inputs.

\begin{table}[t]
    \centering
    \caption{Full pipeline evaluation results on H36M for varying $N_{\text{req}}$.}
    \label{tab:full_pipeline_results}
    \begin{tabular}{lccc}
        \toprule
        $N_{\text{req}}$ & $\downarrow$ $\mathcal{H}$ invalid [\%] & $\uparrow$ Motion valid [\%] & $\downarrow$ MPJPE [mm] \\
        \midrule
        3 (ours) & \textbf{12.63} & \textbf{74.74} & 55.15 \\
        5 & 13.52 & 73.62 & 55.03 \\
        10 & 14.05 & 72.67 & 54.74 \\
        50 & 16.57 & 68.70 & \textbf{54.14} \\
        \bottomrule
    \end{tabular}
\end{table}

\subsection{Real-World Deployment}
We integrate our human pose pipeline in SARA shield~\cite{thumm_2022_ProvablySafe, thumm_2026_GeneralSafety} and test it in a real-world HRC setting on a Franka Emika robot.
A video of the deployment is available 
\ifthenelse{\boolean{anon}}{
	in the supplementary material.
}{
	at \url{https://youtu.be/qN9ZHieRR7w}.
}
In all tested instances, the robot came to a complete stop before the human operator could reach the robot.

\section{Conclusion} \label{sec:conclusion}

We present a vision-based framework for safe HRC with end-to-end uncertainty propagation, conformal prediction sets, and graceful OOD handling.
Our conformal prediction sets reduce the prediction volume over ISO~13855, while our OOD handling reduces invalid motion predictions.
Future work includes fusion with diverse sensor modalities, establishing real-time guarantees, and PFH$_\text{D}$ evaluation on a larger, more diverse set of motion data.

\ifthenelse{\boolean{anon}}{}{
\section{Acknowledgements}
This work was supported by Toyota Motor Engineering \& Manufacturing North America, Inc. and Volkswagen Group of America, Inc.
}

\bibliographystyle{IEEEtran}
\bibliography{library_cleaned}

\balance

\end{document}